\documentclass[onecolumn, 11pt, letterpaper]{article}

\usepackage[margin = 1in]{geometry}

\usepackage{amsthm}
\usepackage{amssymb}
\usepackage{amsmath}

\usepackage[utf8]{inputenc} 
\usepackage[T1]{fontenc}    
\usepackage{hyperref}       
\usepackage{booktabs}       
\usepackage{nicefrac}       
\usepackage{microtype}      
\usepackage{xcolor}         
\usepackage{multicol}
\usepackage{graphicx}
\usepackage{mathrsfs}
\usepackage{color}
\usepackage{xcolor}
\usepackage{multirow}
\usepackage{paralist}
\usepackage{verbatim}
\usepackage{algorithm}
\usepackage{algorithmic}
\usepackage{galois}
\usepackage{boxedminipage}
\usepackage{accents}
\usepackage{stmaryrd}
\usepackage[bottom]{footmisc}
\definecolor{light-gray}{gray}{0.85}
\usepackage{colortbl}
\usepackage{makecell}
\usepackage{hhline}
\usepackage{bbm}
\usepackage{mathtools}
\usepackage{MnSymbol}
\usepackage{xr}

\usepackage{pifont}
\newtheorem{theorem}{Theorem}
\newtheorem{lemma}{Lemma}
\newtheorem{assumption}{Assumption}
\newtheorem{definition}{Definition}

\newtheorem{proposition}{Proposition}

\newtheorem{remark}{Remark}

\newcommand{\norm}[1]{\left\lVert#1\right\rVert}
\newcommand\inner[2]{\left\langle #1, #2 \right\rangle}

\newcommand{\mc}{\mathcal}
\newcommand{\mbb}{\mathbb}

\newcommand{\mcN}{{\mathcal{N}}}
\newcommand{\mcNk}{\mathcal{N}^\kappa}

\newcommand{\pt}{{\pi_\theta}}
\newcommand{\pti}{{\pi_{\theta_i}^i}}
\newcommand{\ptt}{{\pi_{\theta^t}}}

\newcommand{\hpt}{{\widehat{\pi}_\theta}}
\newcommand{\hpti}{{\widehat{\pi}_{\theta_i}^i}}

\newcommand{\ba}{\begin{array}}
\newcommand{\ea}{\end{array}}

\title{\LARGE \bf
Scalable Multi-Agent Reinforcement Learning with General Utilities
}

\author{Donghao Ying \thanks{Department of Industrial Engineering and Operations Research, University of California at Berkeley. Email: {\tt\small \{donghaoy,yuhao\_ding,lavaei\}@berkeley.edu}} \and
Yuhao Ding$^*$ \and
Alec Koppel \thanks{J.P. Morgan AI Research. Email: {\tt\small alec.koppel@jpmchase.com}} \and
Javad Lavaei$^*$}

\date{}
\begin{document}
\maketitle
\begin{abstract}
We study the scalable  multi-agent reinforcement learning (MARL) with general utilities, defined as nonlinear functions of the team's long-term state-action occupancy measure. 
The objective is to find a localized policy that maximizes the average of the team's local utility functions without the full observability of each agent in the team.
By exploiting the spatial correlation decay property of the network structure, we propose a scalable distributed policy gradient algorithm with shadow reward and localized policy that consists of three steps: (1) shadow reward estimation, (2) truncated shadow Q-function estimation, and (3) truncated policy gradient estimation and policy update. Our algorithm converges, with high probability, to $\epsilon$-stationarity with $\widetilde{\mc{O}}(\epsilon^{-2})$ samples up to some approximation error that decreases exponentially in the communication radius.
This is the first result in the literature on multi-agent RL with general utilities that does not require the full observability.
\end{abstract}


\section{INTRODUCTION}
Many decision-making problems take a form beyond the classic cumulative reward, such as apprenticeship learning \cite{abbeel2004apprenticeship},  diverse skill discovery \cite{eysenbach2018diversity}, pure exploration \cite{hazan2019provably}, and state marginal matching \cite{lee2019efficient}, among others. Such problems can be abstracted
as \textit{reinforcement Learning (RL) with general utilities} \cite{zhang2020variational, zahavy2021reward}, which focus on finding a policy to maximize a nonlinear function of the induced state-action occupancy measure. It generalizes the standard RL in which the objective is only an inner product between the state-action occupancy measure induced by the policy and a policy-independent reward for each state-action pair. 

Beyond the single agent RL, consider the multi-agent problem where different agents need to interact to obtain a favorable outcome by finding a decision policy that maximizes the global accumulation of all agent's general utility. This setting captures a wide range of applications, e.g. epidemics \cite{mei2017dynamics}, social networks \cite{jaques2019social}, finance \cite{lee2002stock}, intelligent transportation \cite{zhang2016control} and wireless communication networks \cite{zocca2019temporal}.
Recently, \cite{zhang2022multi} proposed a new mechanism for cooperation that allows agents to incorporate general utilities for multi-agent RL (MARL)  with common payoffs among agents.
To enable the decentralization of agents’ policies under general utilities, \cite{zhang2022multi} defines local occupancy measure of each agent as a marginalization of the global occupancy measure, and it defines the local general utility of the agent as an arbitrary function of its local occupancy measure. 
Based on these definitions, \cite{zhang2022multi} derives a policy gradient-based algorithm, namely Decentralized Shadow Reward Actor-Critic, where each agent estimates its policy gradient based on local information and communications with its neighbors.

However, their approach assumes the full observability, i.e., each agent  should have access to the  global states and actions of the team. Such assumption has two limitations. 
First, it is expensive and sometimes impossible to communicate with all agents in the team when the size of the team is large. 
In addition,  full observability also implies that the policy and critic networks in this approach depend on the global states and actions of the team, which may be a barrier to effective decentralized implementation in practice. Moreover, even if individual state and action spaces are often small, the size of global state and action spaces can be exponentially large in the number of agents, which can be fundamentally intractable for large numbers of agents \cite{blondel2000survey}. 

To address these issues, we aim to develop a scalable algorithm for multi-agent RL with general utilities without the full observability assumption. 
Inspired by the localization idea proposed in \cite{qu2022scalable}, our work makes the following contributions:
\begin{itemize}
    \item We derive a truncated policy gradient estimator using the shadow reward and the localized policy for MARL with general utilities. We further establish the  approximation error of the proposed truncated policy gradient estimator based on the spatial correlation decay assumptions;
    \item We propose a distributed policy gradient algorithm with shadow reward and localized policy that consists of three pieces: (1) shadow reward estimation, (2) truncated shadow Q-function estimation, and (3) truncated policy gradient estimation and policy update.
    \item We establish that, with high probability, our algorithm requires $\widetilde{\mc{O}}(\epsilon^{-2})$ samples to achieve $\epsilon$-stationarity with the error term $\mc{O}\left(n\phi_0^{2\kappa}\right)$, where $\phi_0 \in (0,1)$, $n$ is the number of agents, $\mcN$ is the set of agents.
\end{itemize}

It is critical to note that the operating hypotheses we require for developing a localized algorithm for MARL are related to, but distinct from \cite{qu2022scalable} in the following sense: we assume the transition dynamics and policies of all agents are globally correlated and the correlation satisfies a spatial decay property. In contrast, the agents are considered to act on their own with their transitions only affected by the nearest neighbors in \cite{qu2022scalable}.

\subsection{Notations}
For a finite set $\mc{S}$, let $\left|\mathcal{S}\right|$ denote its cardinality and let $\operatorname{TV}(\mu,\mu^\prime) := \sup_{A\subset \mc{S}} |\mu(A)-\mu^\prime(A)|$ be the total variation distance between two probability distributions $\mu$ and $\mu^\prime$ on $\mc{S}$.
When the variable $s$ follows the distribution $\xi$, we write it as $s\sim \xi$.
Let $\mbb{E}[\cdot]$ and $\mbb{E}[\cdot\mid \cdot]$, respectively, denote the expectation and conditional expectation of a random variable.
Let $\mbb{R}$ denote the set of real numbers.
For vectors $x$ and $y$, we use $x^\top$ to denote the transpose of $x$ and use $\langle x,y\rangle$ to denote the inner product $x^\top y$.
We use the convention that $\|x\|_1 = \sum_i |x_i|$, $\|x\| := \|x\|_2= \sqrt{\sum_i x_i^2}$, and $\|x\|_\infty = \max_i |x_i|$.
\section{PROBLEM FORMULATION}

Consider an infinite-horizon Markov Decision Process (MDP) over a finite state space $\mc{S}$ and a finite action space $\mc{A}$ with a discount factor $\gamma\in [0,1)$. Let $\xi$ be the initial distribution. A policy $\pi$ is a function that specifies the decision rule of the agent, i.e., the agent takes action $a\in \mc{A}$ with probability $\pi(a\vert s)$ in state $s\in \mc{S}$.
When action $a$ is taken, the transition to the next state $s^\prime$ from state $s$ follows the probability distribution $s^\prime \sim \mbb{P}(\cdot\vert s,a)$.
In standard RL, the objective is to maximize the expected (discounted) cumulative reward, i.e.,
\begin{equation}
\max_{\pi} V^\pi(r):= \mbb{E}\left[\sum_{k=0}^{\infty} \gamma^{k} r\left(s^{k}, a^{k}\right) \bigg| a^k \sim \pi(\cdot \vert s^k), s^{0}\sim \xi \right],
\end{equation}
where $r(\cdot,\cdot)$ denotes the reward function and the expectation is taken over all possible trajectories.
The value function can also be written as $V^\pi(r) = \left\langle r,\lambda^\pi\right\rangle$, where $\lambda^\pi$ is the {\it{discounted state-action occupancy measure}} defined as 
\begin{equation}
\lambda^\pi(s,a)=\sum_{k=0}^{\infty} \gamma^{k} \mathbb{P}\left(s^{k}=s, a^k = a\big\vert a^k \sim \pi(\cdot \vert s^k), s^0 \sim \xi \right), \forall (s,a).
\end{equation}

We consider a more general problem where the objective is to maximize a general function of $\lambda^\pi$, namely
\begin{equation}\label{eq:general prob}
\max_\pi f(\lambda^\pi),
\end{equation}
where $f:\mbb{R}^{|\mc{S}|\cdot|\mc{A}|}\rightarrow \mbb{R}$ can be a possibly nonlinear function.
Such an objective arises in various applications and is commonly referred to as a {\it general utility} \cite{zhang2022multi,zhang2021convergence}.
For instance, in apprenticeship learning \cite{abbeel2004apprenticeship}, the objective is $f(\lambda^\pi) = -\operatorname{dist}(\lambda^\pi,\lambda_e)$, where $\lambda_e$  corresponds to the expert demonstration and $\operatorname{dist}(\cdot,\cdot)$ is a distance function.
In maximum entropy exploration \cite{hazan2019provably}, $f(\cdot)$ refers to the entropy function such that $f(\lambda^\pi) = -\sum_{s}d^\pi(s)\log d^\pi(s)$, where $d^\pi(s) = (1-\gamma)\sum_{a}\lambda^\pi(s,a)$ is the discounted state occupancy measure.

In this work, we study the decentralized version of \eqref{eq:general prob}, where the system is decentralized among a network of agents associated with a graph $\mc{G} = (\mc{N},\mc{E})$ (not densely connected).
The vertex set $\mc{N}=\{1,2,\dots,n\}$ denotes the set of $n$ agents and the edge set $\mc{E}$ prescribes the communication links among agents.
Let $d(i,j)$ be the distance between agents $i$ and $j$ on $\mc{G}$, defined as the length of the shortest path between them.
For $\kappa \geq 0$, we define $\mcNk_i = \{j\in \mcN \vert d(i,j)\leq \kappa\}$ as the set of agents in the neighborhood of radius $\kappa$ of agent $i$, with the shorthand notations $\mcNk_{-i} := \mcN\backslash  \mcNk_i$ and $-i :=\mcN\backslash  \mcN_i^0 = \mcN\backslash\{i\}$.
The details of the decentralization are as follows:
\paragraph{Space Decomposition}The global state and action spaces are the product of local spaces, i.e., $\mc{S}=\mc{S}_1\times \mc{S}_2\times \cdots \times \mc{S}_n$, $\mc{A}=\mc{A}_1\times \mc{A}_2\times \cdots \times \mc{A}_n$, meaning that for every $s\in \mc{S}$ and $a\in \mc{A}$, we can write $s=(s_1,s_2,\dots,s_n)$ and $a = (a_1,a_2,\dots,a_n)$.
For each subset ${\mcN}^\prime\subset \mcN$, we use $(s_{{\mcN}^\prime},a_{{\mcN}^\prime})$ to denote the state-action pair for agents in ${\mcN}^\prime$.
We assume that each agent has direct access to its own states and actions while accessing other agents' information requires communications.

\paragraph{Transition Decomposition}Given the current global state $s$ and action $a$, the local states in the next period are independently generated, i.e.,  $\mbb{P}(s^\prime \vert s,a) = \prod_{i\in \mc{N}} \mbb{P}_i(s_i^\prime \vert s,a)$, $\forall s^\prime \in \mc{S}$, where $\mbb{P}_i$ denotes the local transition probability.

\paragraph{Policy Factorization}The global policy can be decomposed as $\pi(a\vert s) = \prod_{i\in \mc{N}} \pi^{i}\left(a_{i} \vert s_{\mcNk_i}\right), \forall (s,a)$, i.e., given global state $s$, each agent $i$ acts independently according to its local policy $\pi^i$, which depends on the state of agents in $\mcNk_i$. For the policy parameterization, we assume that the local policy of agent $i$ is parameterized by $\theta_i$, and therefore one can write $\pi(a\vert s) = \pi_\theta(a\vert s) = \prod_{i\in \mc{N}}\pi_{\theta_{i}}^{i}\left(a_{i} \vert s_{\mcNk_i}\right)$, where $\theta = (\theta_1,\theta_2,\dots,\theta_n)\in \Theta$ is the global parameter.
\paragraph{Local Utility}For each agent $i$, define its {\it local discounted state-action occupancy measure} as
\begin{equation}
\lambda^\pi_i(s_i,a_i)=\sum_{k=0}^{\infty} \gamma^{k} \mathbb{P}\left(s^{k}_i=s_i, a^k_i = a_i\big\vert a^k \sim \pi(\cdot \vert s^k), s^0 \sim \xi \right), \forall (s_i,a_i),
\end{equation}
which can be viewed as the marginalization of the global occupancy measure, i.e., $\lambda^\pi_i(\hat s_i,\hat a_i) = \sum_{s_i=\hat s_i,a_i=\hat a_i}\lambda^\pi(s,a)$.
Then, the global utility function $f(\cdot)$ can be written as the average of local utilities, i.e., $f(\lambda^\pi) = 1/n\times \sum_{i\in \mc{N}}f_i(\lambda^\pi_i)$, where $f_i:\mbb{R}^{|\mc{S}_i|\cdot|\mc{A}_i|}\rightarrow \mbb{R}$ is a function of the local occupancy measure $\lambda^\pi_i$ and is private to agent $i$.
Thus, under the parameterization $\pi_\theta$, \eqref{eq:general prob} can be rewritten as
\begin{equation}\label{eq:dec_prob}
\max_{\theta\in \Theta} F(\theta),
\text{ where } F(\theta):=f(\lambda^{\pt}) = \frac{1}{n}\cdot\sum_{i\in \mc{N}}f_i(\lambda^\pt_i).
\end{equation}
Finally, we remark that, by choosing all $f_i(\cdot)$ to be linear, \eqref{eq:dec_prob} reduces to standard MARL, where each agent $i$ is associated with a local reward function $r_i:\mc{S}_i\times \mc{A}_i\rightarrow \mbb{R}$ and the global reward is defined as $r(s,a):=1/n\times \sum_{i\in \mc{N}}r_i(s_i,a_i)$.

\section{Truncated Policy Gradient Algorithm with Shadow Reward}\label{sec:alg}
In RL with cumulative reward, the \textit{policy gradient theorem} \cite{sutton1999policy} applies to computing the gradient of the value function:
\begin{equation}
\begin{aligned}
\nabla_\theta V^\pt(r)&=\frac{1}{1-\gamma}\mbb{E}_{s\sim d^\pt,a\sim\pt(\cdot\vert s)}\left[\psi_\theta(a\vert s)\cdot Q^\pt(r;s,a)\right],\\
&=\mbb{E}\left[\sum_{k=0}^\infty \gamma^k\psi_\theta(a^k\vert s^k)\cdot Q^\pt(r;s^k,a^k)\bigg\vert a^k \sim \pt(\cdot \vert s^k), s^0\sim\xi \right],
\end{aligned}
\end{equation}
where $\psi_\theta(\cdot \vert \cdot ):=\nabla_\theta \log \pt(\cdot \vert \cdot )$ denotes the score function, and $Q^\pi(r;s,a)$ is the state-action value function (Q-function) under reward $r(\cdot,\cdot)$, defined as 
\begin{equation}
Q^\pi(r;s,a)=\mbb{E}\left[\sum_{k=0}^{\infty} \gamma^{k} r\left(s^{k}, a^{k}\right) \bigg|  a^k \sim \pi(\cdot \vert s^k), s^{0}=s,a^0=a \right].
\end{equation}
However, for objective \eqref{eq:dec_prob} with general utilities, this elegant result no longer holds. 
Instead, we have the following lemma.
\begin{lemma}\label{lemma:policy_gradient}
For every policy $\pt$, it holds that
\begin{equation}
\nabla_\theta F(\theta) = \frac{1}{1-\gamma}\mbb{E}_{s\sim d^\pt,a\sim\pt(\cdot\vert s)}\left[\psi_\theta(a\vert s)\cdot Q^\pt_f(s,a)\right],
\end{equation}
where $Q^\pt_f(\cdot,\cdot) := Q^\pt(r^\pt;\cdot,\cdot)$ is the {\it shadow Q-function} and $r^\pt:= \nabla_\lambda f(\lambda^\pt)\in \mbb{R}^{|\mc{S}|\times|\mc{A}|}$ is the {\it shadow reward} associated with policy $\pt$.
\end{lemma}
\begin{proof}
For value functions with cumulative reward, we observe the relation $\nabla_\theta V^\pt(r) = \nabla_\theta \left\langle r,\lambda^\pt \right\rangle =  \left\langle r,\nabla_\theta\lambda^\pt \right\rangle$. Thus, by the chain rule, we have that
\begin{equation}
\nabla_\theta F(\theta)=\nabla_\theta f(\lambda^\pt) = \left\langle \nabla_\lambda f(\lambda^\pt),\nabla_\theta\lambda^\pt \right\rangle  = \nabla_\theta V^\pt(r^\pt),
\end{equation}
which completes the proof by the policy gradient theorem.
\end{proof}

In the decentralized formulation \eqref{eq:dec_prob}, for each agent $i$, let $r^\pt_i := \nabla_{\lambda_i} f_i(\lambda^\pt_i)\in \mbb{R}^{|\mc{S}_i|\times|\mc{A}_i|}$ be the local shadow reward, which only depends on the local state and action for a given policy $\pt$, and we define the local shadow Q-function as $Q_{i}^\pt(s,a):=Q^\pt(r^\pt_i;s,a)$. 
Then, it is clear that $r^\pt = 1/n\times \sum_{i\in \mcN}r^\pt_i$ and $Q_{f}^\pt(s,a)  = 1/n\times \sum_{i\in \mcN} Q_{i}^\pt(s,a)$, and the gradient of $F(\theta)$ with respect to agent $i$'s local parameter $\theta_i$ can be written as
\begin{equation}\label{eq:policy_gradient_dec}
\nabla_{\theta_i} F(\theta) = \frac{1}{1-\gamma}\mbb{E}_{s\sim d^\pt, a\sim\pt(\cdot\vert s)}\left[\psi_{\theta_i}(a_i\vert s_{\mcNk_i})\cdot\frac{1}{n}\sum_{j\in \mcN} Q^\pt_j(s,a)\right],
\end{equation}
where we use the policy factorization to derive that $\nabla_{\theta_i} \log \pt (a\vert s) = \nabla_{\theta_i}\log \pti (a_i\vert s_{\mcNk_i})=: \psi_{\theta_i} (a_i\vert s)$, and we refer to $\psi_{\theta_i}(\cdot\vert \cdot)$ as the local score function.
Thus, updating the local parameter $\theta_i$ with the gradient \eqref{eq:policy_gradient_dec} requires knowing the global state and action as well as the shadow Q-functions of all agents, which can be inefficient in large networks due to the communication cost.
In the remainder of the section, we show that an accurate gradient estimator can be designed for all agents while only local communications with neighbors are required under some correlation decay assumptions.

\subsection{Spatial Correlation Decay Assumption}
Following \cite{alfano2021dimension}, we assume that a form of correlation decay property holds for the transition probability \cite{georgii2011gibbs,gamarnik2013correlation}. 
\begin{assumption}\label{assump:decay_trans}
For a matrix $M\in \mbb{R}^{n\times n}$ whose $(i,j)$ entry is defined as
\begin{equation}
M_{i j}=\sup_{s_j, a_j,s_j^{\prime}, a_j^{\prime},s_{-j},a_{-j}} \operatorname{TV}\left(\mbb{P}_i\left(\cdot \vert s_j, s_{-j}, a_j, a_{-j}\right), \mbb{P}_i\left(\cdot \vert s_j^{\prime}, s_{-j}, a_j^{\prime}, a_{-j}\right)\right),
\end{equation}
assume that there exists $\beta\geq 0$ such that 
\begin{equation}
\max _{i \in \mcN} \sum_{j \in \mcN} e^{\beta d(i, j)} M_{i j} \leq \rho,
\end{equation}
with $\rho<1/\gamma$, where $\gamma$ is the discount factor.
\end{assumption}
By definition, the element $M_{ij}$ characterizes the maximum level of impact of agent $j$'s state and action on the local transition probability of agent $i$.
Then, Assumption \ref{assump:decay_trans} mainly requires that such impacts decrease exponentially with respect to the distance between agents.
Such a decay is usually typical in engineered systems with large networks, e.g., in wireless communication where the strength of signals decreases exponentially with the distance \cite{tse2005fundamentals,roberts1975aloha}.

\subsection{Truncated Shadow Q-function}
We first introduce the notion of {\it exponential decay} for Q-functions \cite{qu2022scalable}, which is a form of correlation decay property.
\begin{definition}
For $c\geq 0$ and $\phi \in (0,1)$, the $(c,\phi)$-exponential decay property holds if, for every policy $\pi_\theta$, agent $i$, and state-action pairs $(s,a),(s^\prime,a^\prime)\in \mc{S}\times\mc{A}$ with $s_{\mcNk_i}=s^\prime_{\mcNk_i}$, $a_{\mcNk_i}=a^\prime_{\mcNk_i}$, the local shadow Q-function satisfies
\begin{equation}
\left|Q_i^\pt(s,a)-Q_i^\pt(s^\prime,a^\prime)\right|\leq c\phi^{\kappa}.
\end{equation}
\end{definition}
The exponential decay property holds when the dependency of each agent's local shadow Q-function on other agents' states and actions exponentially decreases with respect to their distances. 
Motivated by \cite{qu2022scalable} and \cite{gamarnik2013correlation}, for every $i$, we define $\widehat{Q}^\pt_{i}:\mc{S}_{\mcNk_i}\times \mc{A}_{\mcNk_i}\rightarrow \mbb{R}$ to be agent $i$'s truncated shadow Q-function, depending only on the states and actions of agents in the neighborhood $\mcNk_i$:
\begin{equation}\label{eq:truncated_Q}
\begin{aligned}
\widehat{Q}^\pt_{i}(s_{\mcNk_i},a_{\mcNk_i}):=Q^\pt_{i}(s_{\mcNk_i},\bar{s}_{\mcNk_{-i}},a_{\mcNk_i},\bar{a}_{\mcNk_{-i}}),
\end{aligned}
\end{equation}
for every $(s_{\mcNk_i},a_{\mcNk_i})\in \mc{S}_{\mcNk_i}\times \mc{A}_{\mcNk_i}$,
where $(\bar{s}_{\mcNk_{-i}},\bar{a}_{\mcNk_{-i}})$ is any fixed state-action pair for the agents in $\mcNk_{-i}$.
That is, the estimator $\widehat{Q}^\pt_{i}(s_{\mcNk_i},a_{\mcNk_i})$ can be viewed as an approximate of the true shadow Q-function $Q^\pt_i(s,a)$ by taking arbitrary values for $(s_{\mcNk_{-i}},a_{\mcNk_{-i}})$.
Compared with $Q^\pt_i$, the estimator $\widehat{Q}^\pt_{i}$ depends on much smaller state and action spaces, and it is thus easy to estimate and store.

When the $(c,\phi)$-exponential decay property holds for Q-functions, it can be intuitively understood that the accuracy of this approximation has the order $\mc{O}(\phi^\kappa)$.
The following lemma shows that, when Assumption \ref{assump:decay_trans} holds and the shadow reward is universally bounded, the exponential decay property is satisfied.
We are thus capable of proving that $\widehat{Q}^\pt_{i}$ is a satisfactory approximation of $Q^\pt_i$.

\begin{lemma}\label{lemma:exponential_decay}
Suppose that Assumption \ref{assump:decay_trans} holds and there exists $ M_f>0$ such that $\|\nabla_{\lambda_i}f_i(\lambda_i^\pt)\|_\infty\leq M_f$, $\forall i\in \mc{V}, \theta\in \Theta$. 
Then, \textbf{\textit{(I)}} the $(c_0,\phi_0)$-exponential decay property holds with $(c_0,\phi_0) = \left(\frac{2\gamma \rho M_f}{1-\gamma \rho},e^{-\beta} \right)$,
\textbf{\textit{(II)}} the truncated shadow Q-function satisfies
$\underset{s,a}{\sup} \left|\widehat{Q}^\pt_{i}(s_{\mcNk_i},a_{\mcNk_i})-Q^\pt_i(s,a)\right|\leq c_0\phi_0^\kappa$.
\end{lemma}
Under the bounded gradient assumption, we can treat the shadow Q-functions as standard Q-functions with bounded reward functions.
We refer the reader to \cite{alfano2021dimension} for the proof of part \textbf{\textit{(I)}} in Lemma \ref{lemma:exponential_decay}.
Then, part \textbf{\textit{(II)}} follows directly from the definition of the exponential decay property.
We note that the set of all possible state-action occupancy measures forms a convex polytope in $\mbb{R}^{|\mc{S}|\times |\mc{A}|}$ and is therefore a compact set.
Thus, requiring the existence of $M_f>0$ in Lemma \ref{lemma:exponential_decay} is not a restrictive assumption and it naturally holds if the gradient $\nabla_\lambda f(\lambda)$ is a continuous mapping on the set of occupancy measures.
We additionally remark that a faster rate of the exponential decay property may be proved under extra assumptions, e.g., mixing properties of the underlying Markov chain \cite{qu2022scalable}.

\subsection{Truncated Policy Gradient Estimator}
In this section, we introduce how the exponential decay property can help design scalable algorithms.

As mentioned earlier, the major challenge in employing the exact policy gradient \eqref{eq:policy_gradient_dec} comes from obtaining the global state-action pairs and the local shadow Q-functions of all agents, which may incur high costs in large networks.
Instead, we consider the following truncated policy gradient estimator:
\begin{equation}\label{eq:truncated_grad}
\begin{aligned}
\widehat{g}_{i}(\theta) = \frac{1}{1-\gamma}\mbb{E}_{s\sim d^\pt, a\sim\pt(\cdot\vert s)}\bigg[{\psi}_{\theta_i}(a_i\vert s_{\mcNk_i}) \cdot\frac{1}{n}\sum_{j\in \mcNk_i}\widehat{Q}^\pt_{j}(s_{\mcNk_j},a_{\mcNk_j})
\bigg],
\end{aligned}
\end{equation}
Compared to the true policy gradient \eqref{eq:policy_gradient_dec}, the estimator $\widehat{g}_{i}(\theta)$ replaces the shadow Q-functions with their truncated estimators. Furthermore, it only uses the truncated Q-functions of agents in $\mcNk_i$.
In the next proposition, we evaluate the approximation error of $\widehat{g}_{i}(\theta)$.
\begin{proposition}\label{prop:truncated_eval}
Let Assumption \ref{assump:decay_trans} hold. Suppose that there exist $M_f,M_\psi>0$ such that $\|\nabla_{\lambda_i}f_i(\lambda_i^\pt)\|_\infty\leq M_f$ and
$\|\psi_{\theta_i}(a_i\vert s_{\mcNk_i})\|\leq M_\psi $, $\forall i\in \mcN, (s,a)\in \mc{S}\times \mc{A}, \theta\in \Theta$.
Then, for all $i\in \mcN, \theta\in \Theta$, we have that
\begin{equation}\label{eq:truncated_eval}
\|\widehat{g}_{i}(\theta)-\nabla_{\theta_i} F(\theta)\|\leq 
\frac{c_0\phi_0^\kappa M_\psi}{1-\gamma}.
\end{equation}
\end{proposition}
\begin{proof}
In this proof, we write $\mbb{E}_{s\sim d^\pt,a\sim\pt(\cdot\vert s)}$ simply as $\mbb{E}$. The difference term in \eqref{eq:truncated_eval} can be expanded as 
\begin{equation}\label{eq:truncated_grad_decompose}
\begin{aligned}
\widehat{g}_{i}(\theta)-\nabla_{\theta_i} F(\theta)
&=\dfrac{1}{n(1-\gamma)}{\mbb{E}\left[\psi_{\theta_i}(a_i\vert s_{\mcNk_i})\left(\sum_{j\in \mcNk_i}\widehat{Q}^\pt_{j}(s_{\mcNk_j},a_{\mcNk_j}) - \sum_{j\in \mcN} Q^\pt_j(s,a)\right) \right]}\\
&=\dfrac{1}{n(1-\gamma)} \mbb{E}\left[\psi_{\theta_i}(a_i\vert s_{\mcNk_i})\sum_{j\in \mcN} \left(\widehat{Q}^\pt_{j}(s_{\mcNk_j},a_{\mcNk_j}) - Q^\pt_j(s,a)\right) \right]\\
&\quad-\dfrac{1}{n(1-\gamma)}\mbb{E}\left[\psi_{\theta_i}(a_i\vert s_{\mcNk_i}) \sum_{j\in \mcNk_{-i}}\widehat{Q}^\pt_{j}(s_{\mcNk_j},a_{\mcNk_j}) \right].
\end{aligned}
\end{equation}
Now, we show that the second term above is actually 0. Indeed, for given $s\in \mc{S}$, one can write:
\begin{equation}
\begin{aligned}
&\quad\mbb{E}_{a\sim \pt(\cdot\vert s)}\left[
\psi_{\theta_i}(a_i\vert s_{\mcNk_i}) \sum_{j\in \mcNk_{-i}}\widehat{Q}^\pt_{j}(s_{\mcNk_j},a_{\mcNk_j})
\right]\\
&=\sum_{a}\prod_{k\in \mcN} \pi_{\theta_k}^k(a_k\vert s_{\mcNk_k})\cdot \frac{\nabla_{\theta_i}\pti (a_i\vert s_{\mcNk_i})}{\pti(a_i\vert s_{\mcNk_i})}\cdot \sum_{j\in \mcNk_{-i}}\widehat{Q}^\pt_{j}(s_{\mcNk_j},a_{\mcNk_j})\\
&=\sum_{a} \prod_{k\neq i}\pi_{\theta_k}^k(a_k\vert s_{\mcNk_k})\cdot \nabla_{\theta_i}\pti (a_i\vert s_{\mcNk_i})\cdot \sum_{j\in \mcNk_{-i}}\widehat{Q}^\pt_{j}(s_{\mcNk_j},a_{\mcNk_j})\\
&=\sum_{a_{-i}}\bigg[ \left(\sum_{a_i}\nabla_{\theta_i}\pti (a_i\vert s_{\mcNk_i})\right)\cdot \prod_{k\neq i}\pi_{\theta_k}^k(a_k\vert s_{\mcNk_k})\cdot \sum_{j\in \mcNk_{-i}}\widehat{Q}^\pt_{j}(s_{\mcNk_j},a_{\mcNk_j})\bigg]\\
&=0,
\end{aligned}
\end{equation}
where we expand the expectation and the score function in the first equality. The third equality holds since $j\in \mcNk_{-i}$ implies $i\notin \mcNk_j$, and thus the summation $\sum_{j\in \mcNk_{-i}}\widehat{Q}^\pt_{j}(s_{\mcNk_j},a_{\mcNk_j})$ is irrelevant to $a_i$. 
In the last equality, since $\sum_{a_i}\nabla_{\theta_i}\pti (a_i\vert s_{\mcNk_i}) = \nabla_{\theta_i}\left[\sum_{a_i}\pti (a_i\vert s_{\mcNk_i}) \right] =\nabla_{\theta_i} 1 =0$, we conclude that the whole term is equal to zero.
Therefore, it holds that
\begin{equation}
\begin{aligned}
\|\widehat{g}_{i}(\theta)-\nabla_{\theta_i} F(\theta)\|&=\norm{\dfrac{1}{n(1-\gamma)} \mbb{E}\left[\psi_{\theta_i}(a_i\vert s_{\mcNk_i})\sum_{j\in \mcN} \left(\widehat{Q}^\pt_{j}(s_{\mcNk_j},a_{\mcNk_j}) - Q^\pt_j(s,a)\right) \right]}\\
&\leq \dfrac{1}{n(1-\gamma)}\cdot M_\psi\cdot n \cdot c_0\phi_0^\kappa
\end{aligned}
\end{equation}
where we use Lemma \ref{lemma:exponential_decay} to bound the difference between the truncated shadow Q-functions and true shadow Q-functions.
This completes the proof.
\end{proof}
Proposition \ref{prop:truncated_eval} shows that, the accuracy of the truncated gradient estimator has the order $\mc{O}(\phi_0^\kappa)$, which decreases along with the communication radius $\kappa$.
Thus, it indicates a feasible direction to reduce the communication of agents to their $\kappa$-neighborhoods.

\subsection{Algorithm Design}
In this section, we present our method, Distributed Policy Gradient Algorithm with Shadow Reward, for solving problem \eqref{eq:dec_prob}. The algorithm, summarized in Algorithm \ref{alg:tpg}, consists of the following elements:
\begin{algorithm}[tb]
   \caption{Distributed Policy Gradient Algorithm With Shadow Reward and Localized Policy\label{alg:tpg}}
\begin{algorithmic}[1]
   \STATE {\bfseries Input:} Initial policy $\theta^0$; initial distribution $\xi$; communication radius $\kappa$; step-sizes $\{\eta_\theta^t\}$; batch size $B$; episode length $H$.
   \FOR{iteration $t=0,1,2,\dots$}
   \STATE Sample $B$ trajectories $\tau = \left\{(s^{0}, a^{0}), \cdots, (s^{H-1}, a^{H-1})\right\}$ with length $H$, under policy $\pi_{\theta^t}$, initial distribution $\xi$. Collect them as batch $\mc{B}_t$.
    \STATE Each agent $i$ estimates its local occupancy measure $\lambda^\ptt_i$ through
    \begin{equation}\label{eq:occupancy_estimate}
    \widetilde{\lambda}^t_i = \frac{1}{B} \sum_{\tau \in \mathcal{B}_{t}} \sum_{k=0}^{H-1} \gamma^{k} \cdot \mathbf{e}_i\left(s_{i}^{k}, a_{i}^{k}\right)\in\mbb{R}^{|\mc{S}_i|\times |\mc{A}_i|},
    \end{equation}
    and computes the empirical shadow reward 
    $\widetilde{r}_i^t = \nabla_{\lambda_i}f_i(\widetilde{\lambda}^t_i)$.
    \STATE Each agent $i$ communicates with its neighborhood $\mcNk_i$ and estimate the truncated Q-function under $\widetilde{r}_i^t$, denoted as $\widetilde{Q}^t_{i}$.
    \STATE Each agent $i$ shares $\widetilde{Q}^t_{i}$ with its neighborhood $\mcNk_i$ and estimates the truncated policy gradient through
    \begin{equation}\label{eq:grad_estimate}
    \begin{aligned}
    \widetilde{g}_{i}^t =\frac{1}{B}\sum_{\tau\in \mc{B}_t} \bigg[\sum_{k=0}^{H-1}\gamma^k {\psi}_{\theta_i^t}(a_i^k\vert s_{\mcNk_i}^k)\cdot \frac{1}{n} \sum_{j\in \mcNk_i} \widetilde{Q}^t_{i}(s_{\mcNk_j}^k,a_{\mcNk_j}^k)\bigg].
    \end{aligned}
    \end{equation}
    \STATE Each agent $i$ updates the policy through
    \begin{equation}\label{eq:policy_update}
    \theta_i^{t+1} = \theta_i^{t}+\eta_\theta^t\cdot  \widetilde{g}_{i}^t.
    \end{equation}
   \ENDFOR
\end{algorithmic}
\end{algorithm}

\paragraph{Shadow Reward Estimation (lines 3-4)} In the beginning of each iteration $t$, the current policy is simulated to generate a batch of $B$ trajectories with length $H$. 
Since the local policy $\pti(\cdot\vert s_{\mcNk_i})$ of each agent $i$ only depends on the states of $\mcNk_i$, the process of trajectory sampling is comply with the communication requirement.
Then, using local state-action information, each agent $i$ forms an estimation $\widetilde{\lambda}^t_i$ for its local occupancy measure through \eqref{eq:occupancy_estimate}, where we define $\mathbf{e}_i\left(s_i,a_i\right)\in \mbb{R}^{|\mc{S}_i|\times|\mc{A}_i|}$ as a vector with its $(s_i,a_i)$-th entry equal to one and other entries equal to zero.
Finally, the empirical shadow reward is computed via $\widetilde{r}_i^t = \nabla_{\lambda_i}f_i(\widetilde{\lambda}^t_i)$.

\paragraph{Truncated Shadow Q-function Estimation (line 5)} 
In the next stage, each agent $i$ takes $\widetilde{r}_i^t$ as their reward function (pretending that to be the true shadow reward) and communicates with its neighborhood $\mcNk_i$ to estimate the truncated shadow Q-function $ \widehat{Q}^t_{i}$.
We do not specify the estimation process and allow the use of any existing approach for Q-function evaluation as long as it satisfies the error bound required for the theoretical analysis in Section \ref{sec:convergence} (see Assumption \ref{assump:Q_estimation}). For example, one can use the Temporal difference (TD) learning \cite{sutton1988learning}, which is a model-free method for estimating the Q-function.
In TD-learning, all agents iteratively update their estimations along a common trajectory $\tau = \left\{(s^{0}, a^{0}), \cdots, (s^{H-1}, a^{H-1})\right\}$ under policy $\ptt$. For every new global state-action pair $(s^k,a^k)$, the TD-learning updates the current estimation $\widetilde{Q}^t_{i}$ through
\begin{equation}\label{eq:td_learning}
\begin{aligned}
\widetilde{Q}^t_{i}(s^{k-1}_{\mcNk_i},a^{k-1}_{\mcNk_i}) &\leftarrow (1-\eta_Q^{k-1})\widetilde{Q}^t_{i}(s^{k-1}_{\mcNk_i},a^{k-1}_{\mcNk_i}) + \eta_Q^{k-1}\big[\widetilde{r}_i^t(s_i^{k-1},a_i^{k-1})+\gamma \widetilde{Q}^t_{i}(s^{k}_{\mcNk_i},a^{k}_{\mcNk_i}) \big],\\
\widetilde{Q}^t_{i}(s_{\mcNk_i},a_{\mcNk_i}) &\leftarrow  \widetilde{Q}^t_{i}(s_{\mcNk_i},a_{\mcNk_i}), \text{ for }(s_{\mcNk_i},a_{\mcNk_i})\neq (s^{k-1}_{\mcNk_i},a^{k-1}_{\mcNk_i}),
\end{aligned}
\end{equation}
where $\{\eta_Q^k\}$ are the learning step-sizes. As shown in \cite[Theorem 5]{qu2022scalable}, the above procedure exhibits an error rate of $\mc{O}(1/\sqrt{H})$ under a local exploration assumption. Together with the error induced by the empirical shadow reward, this implies $\|\widetilde{Q}^t_{i}- \widehat{Q}^t_{i}\|_\infty = \mc{O}(1/\sqrt{H} + \|\widetilde{r}_i^t-{r}_i^t\|_\infty)$.
Besides the TD-learning, one can also deploy other model-free or model-based estimators depending on the sampling mechanisms, e.g., \cite{li2020sample,gheshlaghi2013minimax}. 

\paragraph{Truncated Policy Gradient Estimation and Policy Update (lines 6-7)}
At the final stage, every agent $i$ exchanges their estimation $\widetilde{Q}^t_{i}$ with the neighborhood $\mcNk_i$ and evaluates the truncated policy gradient \eqref{eq:truncated_grad} through \eqref{eq:grad_estimate}.
The new policy is obtained by performing a policy gradient ascent with the estimated gradient $\widetilde{g}_{i}^t$.

\begin{remark}
In contrast to a major line of MARL research, e.g., \cite{zhang2022multi,zeng2022learning}, full observability is not required for executing Algorithm \ref{alg:tpg}, i.e., the agents do not need have access to the global information, including the global state and action. Instead, for the specified communication radius $\kappa$, each agent $i$ needs to communicate with its neighborhood $\mcNk_i$ to sample trajectories, estimate its local shadow Q-function, and estimate its truncated policy gradient.
\end{remark}

\section{CONVERGENCE ANALYSIS}\label{sec:convergence}
In this section, we analyze the convergence behavior of Algorithm \ref{alg:tpg}.
We first summarize the additional technical assumptions required, among which some have appeared in the previous section.
\begin{assumption}\label{assump:utility}
Let $\Lambda$ be the set of all possible occupancy measures $\lambda$.
The utility function $f(\cdot)$ satisfies: \textbf{\textit{(I)}} $\exists M_f>0$ such that $\|\nabla_{\lambda_i} f_i(\lambda_i)\|_\infty\leq M_f$, $\forall i\in \mcN$ and $ \lambda\in \Lambda$.
\textbf{\textit{(II)}} $\exists L_\lambda$ such that $\|\nabla_{\lambda_i}f_i(\lambda_i) -\nabla_{\lambda_i}f_i(\lambda_i^\prime)\|_\infty \leq L_\lambda\|\lambda_i-\lambda_i^\prime\|$, $\forall i\in \mcN$ and $ \lambda, \lambda^\prime\in \Lambda$.
\end{assumption}

\begin{assumption}\label{assump:policy}
The parameterized policy $\pi_\theta$ and the associated occupancy measure $\lambda^\pt$ satisfy:
\textbf{\textit{(I)}} $\exists M_\psi>0$ such that the score function $\|\psi_{\theta_i}(a_i\vert s_{\mcNk_i})\|\leq M_\psi $, $\forall i\in\mcN$, $(s,a)\in \mc{S}\times\mc{A}$, $\theta\in \Theta$. 
\textbf{\textit{(II)}} $\exists L_\theta>0$ such that the utility function $F(\theta)=f(\lambda^\pt)$ is $L_\theta$-smooth with respect to $\theta$.
\end{assumption}
Besides the bounded gradient and the bounded score function assumptions, we additionally assume that the utility function $f_i(\lambda_i^\pt)$ is smooth with respect to both the occupancy measure $\lambda_i$ and the policy $\theta$.
These assumptions are standard in the literature of reinforcement learning with general utilities \cite{hazan2019provably,zhang2022multi, zhang2021convergence,ying2023policy}.

As discussed in Section \ref{sec:alg}, we do not specify the estimation process for the truncated shadow Q-functions.
Instead, we assume that an oracle is used, which produces a bounded-error approximation to the true function.
Let $\widehat{Q}^\pt_{r_i}(\cdot,\cdot)\in \mbb{R}^{|\mc{S}_{\mcNk_i}|\times|\mc{A}_{\mcNk_i}|}$ be the $\kappa$-truncated local Q-function under reward $r_i\in \mbb{R}^{|\mc{S}_i|\times |\mc{A}_i|}$ for agent $i$.
\begin{assumption}\label{assump:Q_estimation}
For every $i\in \mcN$ and $\theta\in \Theta$, an approximation $\widetilde{Q}^\pt_{r_i}(\cdot,\cdot)$ can be computed for $\widehat{Q}^\pt_{r_i}(\cdot,\cdot)$ such that
\begin{equation}\label{eq:Q_estimation}
\sup_{s_{\mcNk_i},a_{\mcNk_i}} \left|\widetilde{Q}^\pt_{r_i}(s_{\mcNk_i},a_{\mcNk_i})-\widehat{Q}^\pt_{r_i}(s_{\mcNk_i},a_{\mcNk_i})\right|\leq \epsilon_0\|r_i\|_\infty,
\end{equation}
where $\epsilon_0>0$ is the approximation error.
\end{assumption}
Under Assumption \ref{assump:Q_estimation}, we have that the estimator $\widetilde{Q}^t_{i}$ in line 5 of Algorithm \ref{alg:tpg} satisfies $\|\widetilde{Q}^t_{i} - \widehat{Q}^\ptt_{\widetilde r_i^t}\|_\infty \leq \epsilon_0 \|\widetilde r_i^t\|_\infty$.
This can be achieved, for example, with $\mc{O}(1/(\epsilon_0)^2)$ samples by the TD-learning procedure \eqref{eq:td_learning}.

Before analyzing the convergence of Algorithm \ref{alg:tpg}, we first present a few auxiliary results, which evaluate the estimators $\widetilde{\lambda}^t_i$, $\widetilde{r}_i^t$, $\widetilde{Q}^t_{i}$, and $\widetilde{g}_{i}^t$.
\begin{proposition}\label{prop:error}
Let $\delta_0 \in (0,1/(2n))$ be the failure probability.
Under Assumptions \ref{assump:utility}-\ref{assump:Q_estimation}, it holds for every period $t\geq 0$ that 
\begin{enumerate}
    \item[\textbf{\textit{(I)}}] for each agent $i\in \mcN$, with probability $1-\delta_0$
\begin{equation}\label{eq:error_r}
\|\widetilde{\lambda}^t_i - \lambda^\ptt_i\|\leq \epsilon_1(\delta_0),\ \| \widetilde{r}_i^t - r_i^t\|_\infty\leq L_\lambda \epsilon_1(\delta_0).
\end{equation}
\item[\textbf{\textit{(II)}}] with probability $1-n\delta_0$
 \begin{equation}\label{eq:error_q}
\|\widetilde{Q}^t_{i}-\widehat{Q}^\ptt_{i} \|_\infty \leq \epsilon_0M_f + \frac{L_\lambda \epsilon_1(\delta_0)}{1-\gamma},\ \forall i\in \mcN.
 \end{equation}
 \item[\textbf{\textit{(III)}}] with probability $1-2n\delta_0$
\begin{equation}\label{eq:error_grad}
\|\widetilde{g}_{i}^t- \widehat{g}_{i}(\theta^t)\|\leq \epsilon_{2,i}(\delta_0),\ \forall i\in \mcN,
\end{equation}
where 
\begin{align}\label{subeq:epsilon_1}
\epsilon_1(\delta_0) = \sqrt{\frac{4+2\gamma^{2H}B-16\log\delta_0}{(1-\gamma)^2B}},\quad 
\epsilon_{2,i}(\delta_0) =  \frac{|\mcNk_i|}{n}\mc{O}\left(\epsilon_0+\sqrt{\frac{\log(1/\delta_0)}{B}}+\gamma^H\right).
\end{align}
\end{enumerate}
\end{proposition}

\begin{proof}
We refer the reader to \cite[Appendix D.1]{zhang2022multi} for the proof of part \textbf{\textit{(I)}}.
For part \textbf{\textit{(II)}}, we first recall that $\widetilde{Q}^t_{i}$ is a sample-based estimation for the truncated local Q-function $\widehat{Q}^\ptt_{\widetilde r_i^t}$.
By combining the error bound \eqref{eq:error_r} with Assumptions \ref{assump:utility}  and \ref{assump:Q_estimation}, we obtain that with probability $1-\delta_0$
\begin{equation}\label{eq:error_q_inter}
\begin{aligned}
\|\widetilde{Q}^t_{i}-\widehat{Q}^\ptt_{i} \|_\infty
&\leq \|\widetilde{Q}^t_{i}-\widehat{Q}^\ptt_{\widetilde r_i^t} \|_\infty + 
\|\widehat{Q}^\ptt_{\widetilde r_i^t}-\widehat{Q}^\ptt_{i} \|_\infty\\
&\leq \epsilon_0\|\widetilde r_i^t\|_\infty + \frac{\| \widetilde{r}_i^t - r_i^t\|_\infty}{1-\gamma}\\
&\leq \epsilon_0M_f + \frac{L_\lambda \epsilon_1(\delta_0)}{1-\gamma}.
\end{aligned}
\end{equation}
By applying the union bound, we have that with probability $1-n\delta_0$, the bound \eqref{eq:error_q_inter} holds for all agents $i\in \mcN$.

For part \textbf{\textit{(III)}}, let $\mc{F}_t$ denote the $\sigma$-algebra generated by all the trajectories in $\mc{B}_t$ sampled at the $t$-th iteration and let
\begin{equation}
\widehat g_i^t := \frac{1}{B}\sum_{\tau\in \mc{B}_t} \bigg[\sum_{k=0}^{H-1}\gamma^k {\psi}_{\theta_i^t}(a_i^k\vert s_{\mcNk_i}^k) \frac{1}{n} \sum_{j\in \mcNk_i} \widehat{Q}^\ptt_{i}(s_{\mcNk_j}^k,a_{\mcNk_j}^k)\bigg],
\end{equation}
which differs from $\widetilde{g}_{i}^t$ only in the Q-function term.
Next, we derive the bound \eqref{eq:error_grad} through the decomposition
\begin{equation}\label{eq:grad_erro_decomposition}
\begin{aligned}
\widetilde{g}_{i}^t- \widehat{g}_{i}(\ptt) = \left(\widetilde{g}_{i}^t- \widehat{g}_{i}^t\right) + \left(\widehat{g}_{i}^t- \mbb{E}\left[\widehat{g}_{i}^t\vert \mc{F}_t\right]\right)+ \left(\mbb{E}\left[\widehat{g}_{i}^t\vert \mc{F}_t\right]- \widehat{g}_{i}(\theta^t)\right).
\end{aligned}
\end{equation}
For the first difference, it holds from part \textbf{\textit{(II)}} and Assumption \ref{assump:policy} that with probability $1-n\delta_0$.
\begin{equation}\label{eq:mg1}
\begin{aligned}
\|\widetilde{g}_{i}^t- \widehat{g}_{i}^t\|^2
&\leq \left[\frac{1}{1-\gamma}\cdot M_\psi\cdot \frac{|\mcNk_i|}{n} \right]^2\left(\epsilon_0M_f + \frac{L_\lambda \epsilon_1(\delta_0)}{1-\gamma}\right)^2\\
&=\frac{|\mcNk_i|^2M_\psi^2}{n^2(1-\gamma)^2}\left(\epsilon_0M_f + \frac{L_\lambda \epsilon_1(\delta_0)}{1-\gamma}\right)^2 =: C_{1,i},
\end{aligned}
\end{equation}
Then, we bound the second term in \eqref{eq:grad_erro_decomposition}.
For a trajectory $\tau$ and $k_1, k_2\geq 0$, we define that
\begin{equation}
G^t_i(\tau_{k_1}^{k_2}) := \sum_{k=k_1}^{k_2}\gamma^k {\psi}_{\theta_i^t}(a_i^k\vert s_{\mcNk_i}^k) \cdot \frac{1}{n} \sum_{j\in \mcNk_i} \widehat{Q}^\ptt_{i}(s_{\mcNk_j}^k,a_{\mcNk_j}^k).  
\end{equation}
It is clear from definition that $\widehat{g}_{i}^t = 1/B\cdot \sum_{\tau\in \mc{B}_t} G^t_i(\tau_0^{H-1})$.
By Assumptions  \ref{assump:utility} and \ref{assump:policy}, it holds that
\begin{equation}
\begin{aligned}
\mbb{E}\left[ \|G^t_i(\tau_0^{H-1})\|^2\vert \mc{F}_t \right]\leq 
\frac{|\mcNk_i|^2M_f^2M_\psi^2}{n^2(1-\gamma)^4} 
\end{aligned}
\end{equation}
Thus, by \cite[Lemma 18]{kohler2017sub}, we have that with probability $1-\delta_0$
\begin{equation}\label{eq:mg2}
\left\|\widehat{g}_{i}^t - \mbb{E}\left[\widehat{g}_{i}^t\vert \mc{F}_t\right]\right\|^2 \leq \frac{(2-8\log \delta_0)|\mcNk_i|^2M_f^2M_\psi^2}{n^2(1-\gamma)^4B} =:C_{2,i}.
\end{equation}
Finally, to bound the third term in \eqref{eq:grad_erro_decomposition}, we derive that
\begin{equation}\label{eq:mg3}
{
\begin{aligned}
\|\mbb{E}\left[\widehat{g}_{i}^t\vert \mc{F}_t\right]- \widehat{g}_{i}(\theta^t)\|^2
&=\left\| \mbb{E} \left[G^t_i(\tau_0^{H-1})\vert \ptt,\xi\right] - \widehat{g}_{i}(\theta^t)\right\|^2\\
&=\left\| \mbb{E} \left[G^t_i(\tau_0^\infty)\vert \ptt,\xi\right] - \widehat{g}_{i}(\theta^t) - \mbb{E} \left[G^t_i(\tau_{H}^\infty)\vert \ptt,\xi\right]\right\|^2\\
&=\left\|\mbb{E} \left[G^t_i(\tau_{H}^\infty)\vert \ptt,\xi\right]\right\|^2\\
&\leq \frac{\gamma^{2H}|\mcNk_i|^2M_\psi^2M_f^2}{n^2(1-\gamma)^4} =: C_{3,i}.
\end{aligned}
}
\end{equation}
where we use the fact that $G^t_i(\tau_0^\infty)$ is an unbiased estimator for $\widehat{g}_{i}(\theta^t)$ in the third equality.
The inequality in the last line follows from Assumptions \ref{assump:utility} and \ref{assump:policy}.

Putting \eqref{eq:grad_erro_decomposition}, \eqref{eq:mg1}, \eqref{eq:mg2}, \eqref{eq:mg3} together,
we have that
\begin{equation}\label{eq:mg_together}
\begin{aligned}
\|\widetilde{g}_{i}^t- \widehat{g}_{i}(\theta^t)\|^2 &\leq 3\bigg[\|\widetilde{g}_{i}^t- \widehat{g}_{i}^t\|^2 + \|\widehat{g}_{i}^t- \mbb{E}\left[\widehat{g}_{i}^t\vert \mc{F}_t\right]\|^2+\|\mbb{E}\left[\widehat{g}_{i}^t\vert \mc{F}_t\right]- \widehat{g}_{i}(\theta^t)\|^2\bigg]\\
&\leq 3(C_{1,i}+C_{2,i}+C_{3,i})\\ 
&=\frac{|\mcNk_i|^2}{n^2}\mc{O}\left(\epsilon_0^2+\epsilon_1(\delta_0)^2+\frac{\log(1/\delta_0)}{B}+\gamma^{2H}\right)\\
&=\frac{|\mcNk_i|^2}{n^2}\mc{O}\left(\epsilon_0^2+\frac{\log(1/\delta_0)}{B}+\gamma^{2H}\right),
\end{aligned}
\end{equation}
where we use the definition of $\epsilon_1(\delta_0)$ in \eqref{subeq:epsilon_1}.
Finally, we note that, by the union bound, \eqref{eq:mg2} holds for all agents $i\in \mcN$ with probability $1-n\delta_0$.
Since \eqref{eq:mg1} holds with probability $1-n\delta_0$ and \eqref{eq:mg3} is deterministic, we conclude that, with probability $1-2n\delta_0$, the bound \eqref{eq:mg_together} holds for all agents.
The proof is completed by taking $\epsilon_{2,i}(\delta_0) := \sqrt{3(C_{1,i}+C_{2,i}+C_{3,i})}$.
\end{proof}

Proposition \ref{prop:error} evaluates the accuracy of the estimation for the truncated policy gradient.
Together with Proposition \ref{prop:truncated_eval}, this provides a probabilistic upper bound for the gradient estimation error $\|\widetilde g^t_i-\nabla_{\theta_i}F(\theta^t)\|$, which we will use to prove the convergence of Algorithm \ref{alg:tpg} in the following theorem.
\begin{theorem}\label{thm:1}
Suppose that Assumptions \ref{assump:decay_trans}-\ref{assump:Q_estimation} hold and the step-sizes satisfy $\eta_\theta^t\leq 1/(4L_\theta)$, $\forall t\geq 0$.
For every $T>0$, let $\delta_0 = \delta/(2nT)$, where $\delta\in(0,1)$ is the failure probability.
Then, with probability $1-\delta$, it holds that
\begin{equation}\label{eq:iteration_complexity}
\frac{\sum_{t=0}^{T-1} \eta_\theta^t\left\|\nabla_{\theta} F(\theta^t)\right\|^2}{\sum_{t=0}^{T-1}\eta_\theta^t}\leq \frac{4\left(F(\theta^T)-F(\theta^0)\right)}{\sum_{t=0}^{T-1}\eta_\theta^t}+3\Delta(\delta_0),
\end{equation}
where
\begin{equation}\label{eq:Delta_mag}
\Delta(\delta_0) =\mc{O}(n\phi_0^{2\kappa})+ \sum_{i\in \mcN}\frac{|\mcNk_i|^2}{n^2}\mc{O}\left(\epsilon_0^2+\frac{\log(1/\delta_0)}{B}+\gamma^{2H} \right).
\end{equation}
\end{theorem}
\begin{proof}
By the smoothness of $F(\theta)$ (Assumption \ref{assump:policy}), when the step-size satisfies $\eta_\theta^t \leq 1/(4L_\theta)$,
the policy update \eqref{eq:policy_update} implies
\begin{equation}\label{eq:ascent}
\begin{aligned}
F(\theta^{t+1}) - F(\theta^t)
&\geq \sum_{i\in \mcN}\left[\inner{\nabla_{\theta_i} F(\theta^t)}{\eta_\theta^t \widetilde{g}_{i}^t}-\frac{L_\theta}{2}\|\eta_\theta^t  \widetilde{g}_{i}^t\|^2\right]\\
&= \sum_{i\in \mcN} \bigg[ \eta_\theta^t\inner{\nabla_{\theta_i} F(\theta^t)}{\nabla_{\theta_i} F(\theta^t)-(\nabla_{\theta_i} F(\theta^t)-\widetilde{g}_{i}^t)}\\
&\quad-\frac{L_\theta}{2}(\eta_\theta^t)^2\left\|\nabla_{\theta_i} F(\theta^t)-(\nabla_{\theta_i} F(\theta^t)-\widetilde{g}_{i}^t)\right\|^2\bigg]\\
&\geq  \sum_{i\in \mcN} \bigg[\eta_\theta^t\left\|\nabla_{\theta_i} F(\theta^t)\right\|^2 - \frac{ \eta_\theta^t}{2}\left(\left\|\nabla_{\theta_i} F(\theta^t)\right\|^2 + \left\|\nabla_{\theta_i} F(\theta^t)-\widetilde{g}_{i}^t\right\|^2\right)\\
&\quad-L_\theta (\eta_\theta^t)^2 \left(\left\|\nabla_{\theta_i} F(\theta^t)\right\|^2 + \left\|\nabla_{\theta_i} F(\theta^t)-\widetilde{g}_{i}^t\right\|^2 \right)\bigg]\\
\end{aligned}
\end{equation}
where we apply the basic inequality $2\inner{a}{b}\leq \|a+b\|^2/2\leq \|a\|^2+\|b\|^2 $ in the last inequality.
By rearranging the terms in \eqref{eq:ascent} and using the condition $\eta_\theta^t \leq 1/(4L_\theta)$, we have that
\begin{equation}
F(\theta^{t+1}) - F(\theta^t)\geq\frac{\eta_\theta^t}{4}\left\|\nabla_{\theta} F(\theta^t)\right\|^2-\frac{3\eta_\theta^t}{4}\sum_{i\in \mcN}\left\|\nabla_{\theta_i} F(\theta^t)-\widetilde{g}_{i}^t\right\|^2,
\end{equation}
which further implies that
\begin{equation}\label{eq:to_sum_up}
\eta_\theta^t\left\|\nabla_{\theta} F(\theta^t)\right\|^2\leq 4\left(F(\theta^{t+1}) - F(\theta^t)\right) + 3\eta_\theta^t\sum_{i\in \mcN}\left\|\nabla_{\theta_i} F(\theta^t)-\widetilde{g}_{i}^t\right\|^2.
\end{equation}
By Propositions \ref{prop:truncated_eval} and \ref{prop:error}, with probability $1-\delta_0$, it holds that
\begin{equation}\label{eq:Delta}
\begin{aligned}
\sum_{i\in \mcN}\left\|\nabla_{\theta_i} F(\theta^t)-\widetilde{g}_{i}^t\right\|^2
&\leq \sum_{i\in \mcN} 2\left(\|\nabla_{\theta_i} F(\theta^t)-\widehat{g}_{i}(\theta^t)\|^2+\| \widehat{g}_{i}(\theta^t)-\widetilde{g}_{i}^t\|^2  \right)\\
&=2\sum_{i\in N} \left[\left(\frac{c_0\phi_0^\kappa M_\psi}{1-\gamma}\right)^2+ \left(\epsilon_{2,i}(\delta_0)\right)^2\right]\\
&=:\Delta(\delta_0).
\end{aligned}
\end{equation}
The relation \eqref{eq:Delta_mag} follows directly from the definition of $\epsilon_{2,i}$ in Proposition \ref{prop:error}, 
Applying the union bound again, we have that \eqref{eq:Delta} holds for all $ t=0,1,\dots, T-1$ with probability $1-\delta$, where $\delta = (2nT)\delta_0$.
Thus, by substituting \eqref{eq:Delta} into \eqref{eq:to_sum_up} and summing over $ t=0,1,\dots, T-1$, we conclude that with probability $1-\delta$
\begin{equation}
\begin{aligned}
\frac{\sum_{t=0}^{T-1} \eta_\theta^t\left\|\nabla_{\theta} F(\theta^t)\right\|^2}{\sum_{t=0}^{T-1}\eta_\theta^t}
&\leq \frac{\sum_{t=0}^{T-1}4\left(F(\theta^{t+1})-F(\theta^t)\right)}{\sum_{t=0}^{T-1}\eta_\theta^t}+\frac{\sum_{t=0}^{T-1}3\eta_\theta^t\Delta(\delta_0)}{\sum_{t=0}^{T-1}\eta_\theta^t}\\
&=\frac{4\left(F(\theta^T)-F(\theta^0)\right)}{\sum_{t=0}^{T-1}\eta_\theta^t}+3\Delta(\delta_0),
\end{aligned}
\end{equation}
which completes the proof.
\end{proof}

Under constant step-sizes $\eta_\theta^t \equiv \eta_\theta$, the bound \eqref{eq:iteration_complexity} becomes
\begin{equation}
\frac{1}{T}\sum_{t=0}^{T-1}\left\|\nabla_{\theta} F(\theta^t)\right\|^2\leq \frac{4\left(F(\theta^T)-F(\theta^0)\right)}{\eta_\theta T} + 3\Delta(\delta_0),
\end{equation}
which implies an $\mc{O}(1/T)$ iteration complexity with the approximation error $3\Delta(\delta_0)$.
As shown in \eqref{eq:Delta_mag}, the constant $\Delta(\delta_0)$ will be small when the rate of spatial correlation decay is fast, the computational error $\epsilon_0$ for Q-functions is small, and enough samples are used to estimate the local occupancy measure.
Notably, when the size of $\kappa$-neighborhood $|\mcNk_i|$ is relatively small for all agents compared to the total number of agents $n$, the term $\sum_{i\in \mcN}{|\mcNk_i|^2}/{n^2}$ approaches $\mc{O}(1/n)$ and $\Delta(\delta_0) = \mc{O}(n\phi_0^{2\kappa}) $ approximately holds.

Suppose that an $\mc{O}(1/(\epsilon_0)^2)$ oracle is used for the truncated Q-function estimation (line 5 in Algorithm \ref{alg:tpg}), i.e., the approximation \eqref{eq:Q_estimation} is achieved with $\mc{O}(1/(\epsilon_0)^2)$ samples.
We analyze the sample complexity of Algorithm \ref{alg:tpg} to compute an $\epsilon$-stationary point.
\begin{theorem}\label{thm:2}
Suppose that Assumptions \ref{assump:decay_trans}-\ref{assump:Q_estimation} hold and an $\mc{O}(1/(\epsilon_0)^2)$ oracle is used for the truncated Q-function estimation.
For every $\epsilon>0$ and $\delta\in (0,1)$, let $T = \mc{O}(\epsilon^{-1})$, $\eta_\theta^t \equiv 1/(4L_\theta)$, $\epsilon_0 = \sqrt{\epsilon}$, $\delta_0 = \delta/(2nT)$, batch size $B=\mc{O}\left(\log(1/\delta_0)\epsilon^{-1}\right)$, episode length $H = \mc{O}\left(\log(1/\epsilon) \right)$.
Then, with probability $1-\delta$, it holds that
\begin{equation}\label{eq:sample_complexity}
\frac{1}{T}\sum_{t=0}^{T-1}\left\|\nabla_{\theta} F(\theta^t)\right\|^2 = \mc{O}\left(\epsilon+ n\phi_0^{2\kappa} \right).
\end{equation}
The total number of samples required is $\widetilde{\mc{O}}(\epsilon^{-2})$.
\end{theorem}
\begin{proof}
The $\epsilon$-stationarity \eqref{eq:sample_complexity} follows directly from \eqref{eq:iteration_complexity} and \eqref{eq:Delta_mag} in Theorem \ref{thm:1}.
In every iteration, $B\times H = \widetilde{\mc{O}}(\epsilon^{-1})$ samples are used to estimate the occupancy measure and compute the empirical shadow reward, (perhaps another) $\mc{O}(1/\epsilon_0^2) = \mc{O}(\epsilon^{-1})$ samples are used to estimate the truncated Q-function.
Since there are $T = \mc{O}(\epsilon^{-1})$ iterations, the total number of samples used is $\widetilde{\mc{O}}(\epsilon^{-2})$.
\end{proof}

As discussed in Section \ref{sec:alg}, the TD-learning procedure \eqref{eq:td_learning} is an $\mc{O}(1/(\epsilon_0)^2)$ oracle for the truncated Q-function estimation with high probability.
Below, we provide two further remarks.
\begin{remark}[Global Optimality]
Suppose that the utility function $f(\lambda)$ is concave in $\lambda$, which generalizes the linear objective for standard RL.
If the policy parameterization satisfies \cite[Assumption 5.11]{zhang2021convergence}, then problem \eqref{eq:dec_prob} does not have spurious local solutions.
Thus, the error bound \eqref{eq:iteration_complexity} implies convergence to global optimality.
\end{remark}

\begin{remark}
The communication radius $\kappa$ plays an important role in both Theorems \ref{thm:1} and \ref{thm:2}. As $\kappa$ increases, the term $\phi_0^{2\kappa}$ decreases, yet the size of the $\kappa$-neighborhood $|\mcNk_i|$ increases, making the constant $\sum_{i\in \mcN}{|\mcNk_i|^2}/{n^2}$ increase.
Also, the increase of $|\mcNk_i|$ will amplify the communication cost and make the estimation of truncated Q-functions less efficient.
Thus, finding a good balance is important in determining $\kappa$.
\end{remark}

\begin{remark}
In this work, we focus on the policy search in a class of localized policies, where each local policy $\pti(a_i\vert s_{\mcNk_i})$ only depends on the states of agents in $\mcNk_i$.
It is possible to relax this ``hard'' requirement to a ``soft'' requirement. 
Below, we briefly describe the idea of this extension.

Consider the factorization $\pi_\theta(a\vert s) = \prod_{i\in \mc{N}}\pi_{\theta_{i}}^{i}\left(a_{i} \vert s\right)$, where each local policy $\pti$ depends on the global state $s$. We assume that a form of spatial correlation decay property holds for the local policy $\pti(a_i\vert s)$ and the associated local score function $\psi_{\theta_i}(a_i\vert s)=\nabla_{\theta_i}\log\pti(a_i\vert s)$, such that
\begin{equation}\label{eq:condition_decay}
\begin{aligned}
\sup_{s,s^\prime_{\mcNk_{-i}}} &\operatorname{TV}\left(\pti(\cdot\vert s_{\mcNk_{i}},s_{\mcNk_{-i}}),\pti(\cdot\vert s_{\mcNk_{i}},s^\prime_{\mcNk_{-i}}) \right)\leq c_1\phi_1^{\kappa},\\
\sup_{s,s^\prime_{\mcNk_{-i}}} &\operatorname{TV}\left(\psi_{\theta_i}(\cdot\vert s_{\mcNk_{i}},s_{\mcNk_{-i}}),\psi_{\theta_i}(\cdot\vert s_{\mcNk_{i}},s^\prime_{\mcNk_{-i}}) \right)\leq c_1\phi_1^{\kappa},
\end{aligned}
\end{equation}
where $c_1\geq 0$ and $\phi_1\in(0,1)$ are two constants. In light of this decay property, we define the induced truncated policy of $\pti$ as (similar to the definition of truncated Q-function)
\begin{equation}
\hpti(a_i\vert s_{\mcNk_i}):= \pti(a_i\vert s_{\mcNk_i},\bar{s}_{\mcNk_{-i}}),\ \forall s_{\mcNk_i}\in \mc{S}_{\mcNk_i},
\end{equation}
where $\bar{s}_{\mcNk_{-i}}$ is any fixed state for the agents in $\mcNk_{-i}$.
Similarly, we define $\widehat \psi(a_i\vert s_{\mcNk_i})$ as the truncated score function.

By using the truncated policy and score function, we can still implement Algorithm \ref{alg:tpg} without violating the observability and communication requirements.
Meanwhile, it is important to quantify the information loss in using the truncated policy as an approximation for the true policy, which depends on the global state.
Specifically, new approximation errors would arise in the trajectory sampling, and then affect the estimation of shadow rewards, shadow Q-functions, and policy gradients.
Under condition \eqref{eq:condition_decay}, the errors in the occupancy measure can be upper-bounded as follows.
\begin{lemma}\label{lemma:occupancy_measure}
Suppose that condition \eqref{eq:condition_decay} holds. Let $\hpt:=\prod_{i\in \mcN} \hpti$ be the induced truncated policy of $\pt$. It holds that
\begin{equation}
\left\|\lambda_i^{\hat{\pi}_\theta}-\lambda_i^{\pi_\theta}\right\|_1 \leq \frac{n c_1 \phi_1^k}{(1-\gamma)^2}, \forall i \in \mathcal{N} .
\end{equation}
\end{lemma}
When Assumption \ref{assump:utility} holds, we can further show that the errors in the empirical shadow rewards and truncated Q-functions have the same order $\mc{O}(\phi_1^\kappa)$ as the occupancy measure.
Together, the same convergence rate and sample complexity as in Theorems \ref{thm:1} and \ref{thm:2} can be proved with an approximation error that has the order $\mc{O}(\phi_0^{2\kappa}+\phi_1^{2\kappa})$, which accounts for the inaccuracies from both the use of truncated policy gradients and the truncated policies. 
We refer the reader to \cite{ying2023scalable} for the proof of Lemma \ref{lemma:occupancy_measure}.

\end{remark}

\section{CONCLUSIONS}
In this paper, we study the scalable  MARL with general utilities, defined as nonlinear functions of the team's long-term state-action occupancy measure.
We propose a scalable distributed policy gradient algorithm with shadow reward and localized policy, which has three steps: (1) shadow reward estimation, (2) truncated shadow Q-function estimation, and (3) truncated policy gradient estimation and policy update.
By exploiting the spatial correlation decay property of the network structure, we rigorously establish the convergence and sample complexity of the proposed algorithm.
Future work includes generalization to the safety-critical setting and considering information asymmetry among the agents.



\section*{ACKNOWLEDGMENT}
This work was supported by grants from ARO, AFOSR, ONR and NSF.

\bibliographystyle{unsrt}
\bibliography{ref}

\end{document}